\documentclass[11pt]{article}
\pdfoutput=1 
\usepackage{algorithm}
\usepackage{algpseudocode}
\usepackage{amsfonts}
\usepackage{amsmath}
\usepackage{amsthm}
\usepackage{authblk}
\usepackage{booktabs}
\usepackage[margin=1in]{geometry}
\usepackage{graphicx}
\usepackage{makecell}
\usepackage{multirow}
\usepackage{siunitx}
\usepackage{pgfplots} 
\pgfplotsset{compat=1.18} 
\usepackage[colorlinks=true,citecolor=blue]{hyperref}
\usepackage{biblatex}
\usepackage{xcolor}
\newcommand{\loss}{\mathcal{L}}
\newcommand{\pder}[2]{\frac{\partial#1}{\partial#2}}
\newcommand{\pdd}[2]{\frac{{\partial}^2#1}{{\partial#2}^2}}

\newcommand{\x}{\mathbf{x}}
\newcommand{\G}{\mathcal{G}}
\newcommand{\numpct}[1]{\pgfmathparse{#1*100}\num[scientific-notation=false, round-mode=figures, round-precision=2]{\pgfmathresult}}

\newtheorem{theorem}{Theorem}
\newtheorem{corollary}{Corollary}[theorem]

\begin{document}

\title{Speeding up and reducing memory usage for scientific machine learning via mixed precision}

\author[1,2,*]{Joel Hayford}
\author[3,*]{Jacob Goldman-Wetzler}
\author[4,*]{Eric Wang}
\author[1,$\dagger$]{Lu Lu}
\affil[1]{Department of Statistics and Data Science, Yale University, New Haven, CT 06511, USA}
\affil[2]{Department of Chemical and Biomolecular Engineering, University of Pennsylvania, Philadelphia, PA 19104, USA}
\affil[3]{Hastings High School, Hastings-on-Hudson, NY 10706, USA}
\affil[4]{WW-P High School North, Plainsboro, NJ 08536, USA}
\affil[*]{These authors contributed equally to this work.}
\affil[$\dagger$]{Corresponding author. Email: lu.lu@yale.edu}

\date{}

\maketitle

\begin{abstract}
Scientific machine learning (SciML) has emerged as a versatile approach to address complex computational science and engineering problems. Within this field, physics-informed neural networks (PINNs) and deep operator networks (DeepONets) stand out as the leading techniques for solving partial differential equations by incorporating both physical equations and experimental data. However, training PINNs and DeepONets requires significant computational resources, including long computational times and large amounts of memory. In search of computational efficiency, training neural networks using half precision (float16) rather than the conventional single (float32) or double (float64) precision has gained substantial interest, given the inherent benefits of reduced computational time and memory consumed. However, we find that float16 cannot be applied to SciML methods, because of gradient divergence at the start of training, weight updates going to zero, and the inability to converge to a local minima. To overcome these limitations, we explore mixed precision, which is an approach that combines the float16 and float32 numerical formats to reduce memory usage and increase computational speed. Our experiments showcase that mixed precision training not only substantially decreases training times and memory demands but also maintains model accuracy. We also reinforce our empirical observations with a theoretical analysis. The research has broad implications for SciML in various computational applications.
\end{abstract}

\section{Introduction} \label{sec:introduction}

In recent years, there has been an enormous increase in using deep learning and other scientific machine learning (SciML) techniques as an alternative to traditional scientific computing methods~\cite{karniadakisphysics-informed2021, brunton2023machine}. The two most common deep learning techniques are data-driven learning~\cite{lu_deeponet_2021, zhubayesian2018, Zhu2023FourierDeepONet} and physics-informed learning~\cite{pinns, zhu_physics-constrained_2019, lu2021deepxde}. Data-driven learning has proven effective in learning not only functions but also nonlinear operators of partial differential equations (PDEs)~\cite{karniadakisphysics-informed2021}. Unlike traditional deep learning methods, deep neural operators, such as deep operator networks (DeepONets)~\cite{lu_deeponet_2021, FAIR-comparison} and Fourier neural operators~\cite{li_fourier_2021}, are designed to learn operators between infinite-dimensional function spaces. Physics-informed learning, such as physics-informed neural networks (PINNs)~\cite{pinns, lu2021deepxde} and their extensions~\cite{pang_fpinns_2019, yu_gradient-enhanced_2022, wu_comprehensive_2023, Zhang2019Quantifying}, have shown promising applications in computational science and engineering~\cite{chen_physics-informed_2020, hard_constraints, daneker_systems_2023, yazdani_systems_2020}. PINNs use knowledge of the underlying PDEs, enabling them to reconstruct functions that satisfy these equations with less data. PINNs have been used to effectively solve various inverse and forward problems in ultrasound testing of materials~\cite{materials}, heat transfer problems~\cite{heat_transfer}, complex fluid modeling~\cite{fluids}, and biomedicine~\cite{biomedicine}. 

Despite the aforementioned success, training PINNs and DeepONets has many challenges. One significant challenge is that the network training usually requires long computational time and a large amount of memory~\cite{yu_gradient-enhanced_2022, FAIR-comparison, Jiang2023FourierMIONet}. For example, training these networks can take several hours or even days, depending on the complexity of the task and the available computational resources~\cite{FAIR-comparison, Jiang2023FourierMIONet}. In addition, its memory-intensive nature often requires high-performance computing facilities~\cite{yu_gradient-enhanced_2022, FAIR-comparison, Jiang2023FourierMIONet}. Possible approaches to overcome these constraints include mixed precision~\cite{micikevicius_mixed_2018}, quantization~\cite{gholami2022a}, pruning~\cite{hanlearning2015}, and distillation~\cite{hinton2015distilling}. 

For deep neural networks, network parameters are commonly processed with the IEEE standard for 32-bit floating point representation, denoted float32. However, modern graphics processing units (GPUs) possess the capability to support half-precision, which employs a 16-bit floating-point format, denoted as float16. Compared to float32, float16 offers benefits such as 50\% reduction in memory and faster arithmetic calculations, which ultimately accelerate the overall computing process~\cite{nvidia}.
Existing work applies float16 to machine learning models mostly in the domain of computer vision~\cite{yun2023defense, jia2018highly}. In particular, these studies have demonstrated substantial computational advantages when using float16 representation, as highlighted in comprehensive reports~\cite{nvidia, yun2023defense, micikevicius_mixed_2018}. 

However, directly replacing float32 with float16 in SciML could be problematic for the network accuracy. 
We show the accuracy losses of using float16 for function regression and PINNs due to difficulties in training. We then determine the causes of the failure of float16 by investigating the values, gradients, and landscapes of training loss. There is a need to develop more effective strategies that can harness the potential of float16 precision in the context of scientific machine learning. 
To address this issue, we apply mixed precision methods to SciML by combining float32 and float16 to maintain the accuracy of float32 while reducing GPU memory and training time by up to 50\%. Finally, we provide a theoretical analysis for the use of mixed precision.

The paper is organized as follows. In Sec.~\ref{sec:methodology}, we introduce two SciML methods, including PINNs and DeepONets. In Sec.~\ref{sec:float16}, we investigate the failure of SciML with float16 by testing function regression and PINNs. We introduce mixed precision methods {in Sec.~\ref{sec:mixed-precision}} and systematically test them on seven problems of PINNs and DeepONets {in Sec.~\ref{sec:mixed-precision: Results}}. In Sec.~\ref{sec:theory}, we theoritically investigate the error of mixed precision. Finally, we conclude the paper in Sec.~\ref{sec:conclusion}.
We implement the experiments using the DeepXDE library~\cite{lu2021deepxde}, and all the codes and data will be available on GitHub at~\url{https://github.com/lu-group/mixed-precision-sciml}.

\section{Methods in scientific machine learning for PDEs} \label{sec:methodology}

In this section, we briefly describe an overview of two impactful SciML methodologies: PINNs and DeepONets. These techniques stand at the forefront of this study, and we will examine them further with the mixed precision method. 

\subsection{Physics-informed neural networks}
We consider a general PDE (parameterized by $\lambda$) in a $d$-dimensional domain $\Omega \subset \mathbb{R}^d$ defined by 
\[\mathcal F[u(\mathbf{x});\lambda] = 0, \quad \mathbf{x} \in \Omega,\] and subject to the initial and boundary conditions (IC and BC)
\[\mathcal B[u(\mathbf x)] = 0, \quad x \in \partial \Omega,\]
where $u(\mathbf x)$ is the solution to the PDE. 

In a PINN, the solution $u(\mathbf x)$ is approximated by a neural network $\hat u(\x; \mathbf \theta)$ with trainable parameters $\mathbf \theta$. The loss function $\loss(\theta)$ is a sum of the PDE loss $\loss_{\text {PDE}}(\theta; \mathcal T_{\text {PDE}})$ evaluated at points $\mathcal T_{\text {PDE}}$, and initial/boundary condition loss $\loss_{\text {IC/BC}}(\theta; \mathcal T_{\text {IC/BC}})$ evaluated at initial/boundary points $\mathcal T_{\text {IC/BC}}$. In particular, the loss is of the form
\[\loss(\theta) = \loss_{\text {PDE}}(\theta; \mathcal T_{\text {PDE}}) + \loss_{\text {IC/BC}}(\theta; \mathcal T_{\text {IC/BC}}),\]
where
\begin{align*}
\loss_{\text {PDE}}(\theta; \mathcal T_\text {PDE}) &= \frac{1}{|\mathcal T_\text {PDE}|} \sum_{\x\in \mathcal T_\text {PDE}} \lVert \mathcal F[\hat u(\x;\theta); \lambda]\rVert_2^2,\\
\loss_{\text {IC/BC}}(\theta; \mathcal T_\text {IC/BC}) &= \frac{1}{|\mathcal T_\text {IC/BC}|} \sum_{\x\in \mathcal T_\text {IC/BC}} \lVert \mathcal B[\hat u(\x;\theta)]\rVert_2^2.
\end{align*}
The partial derivatives in the PDE loss $\loss_\text {PDE}$ are computed using automatic differentiation. In an inverse problem, the architecture is exactly the same, except that $\lambda$ is unknown and is inferred from some solution data by adding an additional loss function.

\subsection{Operator learning with DeepONets}
The DeepONet architecture is designed to effectively learn operators between function spaces based on the universal approximation theorem of neural networks for operators~\cite{lu_deeponet_2021}.   To learn an operator mapping from a function $v$ to another function $u$, i.e., 
$$\G: v \mapsto u,$$
the DeepONet uses two subnets: a branch network and a trunk network. The branch net takes $v$ as the input and returns $[b_1(v), b_2(v), \dots, b_p(v)]$ as output, where $p$ is the number of output neurons. For some $y$ in the domain of $u$, the trunk net takes $y$ as input and outputs $[t_1(y), t_2(y), \dots, t_p(y)]$. The combined DeepONet has output
\[\G(v)(y) = \sum_{k=1}^p b_k(v)t_k(y) + b_0 \]
for some bias $b_0 \in \mathbb R$. 

DeepONet can be trained from a dataset of many pairs of $v$ and $u$ generated by traditional numerical methods. In addition to this data-driven training, the DeepONet can also be trained using  a physics-informed framework, which uses the same loss function as used in PINNs. This adaptation is known as physics-informed DeepONet (PI-DeepONet).

\section {Failure of scientific machine learning with float16} \label{sec:float16}
Implementing low-precision data types in neural networks commonly involves using float16 for both memory storage and computations. Although this method offers the full memory and speed benefits of float16, there may be significant accuracy losses. In this section, we detail two numerical experiments that compare the performance of float16 and float32 on a function regression problem (Sec.~\ref{Function regression}) and a PINN problem for solving a one-dimensional heat equation (Sec.~\ref{heatequation}). We then analyze the accuracy losses and training difficulties encountered in these experiments. 

\subsection{Function regression}
\label{Function regression}
We test a float16 network on a regression problem using a feedforward neural network (FNN). We demonstrate the difficulty in training and the large optimization error for float16, the divergence of the float16 and float32 gradients, and the emergence of the two distinct phases of training when working with float16.

\subsubsection{Problem setup and total error of float16}
\label{prob:regression1}
We consider the function
\[f(x) = x\sin{5x}\]
on the domain $ \Omega = [-1,1]$ for the network to learn.

We trained the networks of float32 or float16 with $16$ equispaced training points (Fig.~\ref{fig:combined-figure-1}) and assessed its performance using $100$ random sampled test data points from $\Omega$. The network is trained for $10000$ iterations, and we recorded the $L^2$ relative error after each training run. To account for individual runs being stochastic, we repeated this training procedure 10 times. Subsequently, we calculated the mean and standard deviation of the $L^2$ relative error across all runs. Following this evaluation, we analyzed the reasons behind the superior performance of the float32 network compared to the float16 models.

The float32 network outperforms the float16 network, and the difference in error between the two network predictions is substantially large (Table~\ref{tab:regressionress}). This discrepancy can be broken down into two distinct error sources: the approximation error (Sec.~\ref{low approximation}) and the optimization error (Sec.~\ref{high optimization}). The approximation error is derived from the limitations in representing values with lower precision, which impacts the network's ability to accurately capture all functions. Optimization error, on the other hand, arises from network training due to reduced numerical precision. In the next subsection, we aim to understand which of the errors contributes mainly to the low accuracy for float16.
\begin{table}[htbp]
    \centering
    \caption{\textbf{$L^2$ relative errors for the two examples in Sec.~\ref{sec:float16}.}}
    \begin{tabular}{c|c c}
         \toprule
          & Regression problem (Sec.~\ref{Function regression}) & Heat equation (Sec.~\ref{heatequation}) \\ 
         \midrule
         Float32 & $\numpct{0.0133}\pm0.3\%$ & $\numpct{0.0035}\pm\numpct{0.0044}\%$ \\
         Float16 & $\numpct{0.0916}\pm\numpct{.0477}\%$ & $\numpct{.017}\pm0.6\%$\\ 
         \bottomrule
    \end{tabular}
    \label{tab:regressionress}
\end{table}

\subsubsection{Low approximation error for float16}
\label{low approximation}
To understand whether a float16 network has the approximability for our target function, we first trained a float32 network, and then we casted the network's weights from float32 to float16.
We reveal that the predictions between the float16 and float32 networks overlap, and the difference in accuracy is negligible (Fig.~\ref{fig:combined-figure-1}). Thus, the large error of float16 in Table~\ref{tab:regressionress} is not due to approximation error, and the source of the discrepancy between the float16 and float32 networks must arise from optimization error during training.

\begin{figure}[htbp]
    \centering
    \includegraphics[width=.5\textwidth]{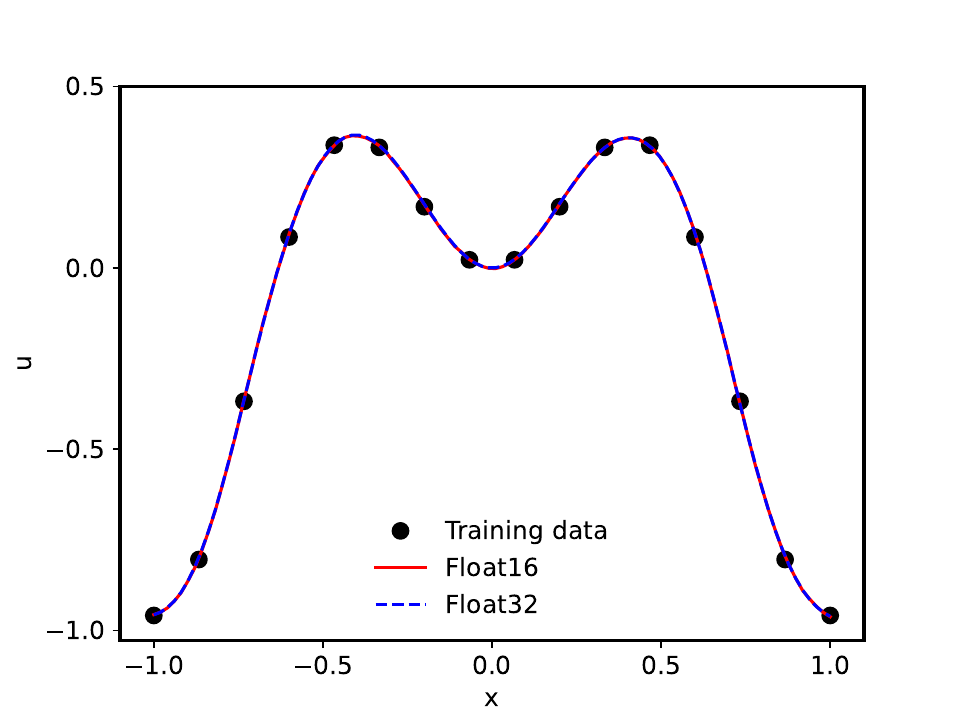}
    \caption{\textbf{Comparison of float32 and float16 for the function regression in Sec.~\ref{Function regression}.} Train a neural network using float32, and then cast the network weights and biases from float32 to float16 after training.}
    \label{fig:combined-figure-1}
\end{figure}

\subsubsection{Training difficulty and large optimization error for float16}
\label{high optimization}
In this section, we look at the reasons behind the high optimization error and how it causes the low accuracy of float16. First, we look at the network initialization step and then analyze the gradients of the networks. We further investigate the loss trajectories of the networks during training from the identical initialization of weights and biases.

\paragraph{Loss gradients at network initialization step.} Although the outputs of the networks are the same at the start of training when they are initialized with the same random seed, their gradients are not. We used the cosine similarity metric to compute the similarity of their gradients at network initialization. The cosine similarity is very small ($0.16 \pm 0.2$). Also, the $L^2$ distance and $L^2$ relative distances are large (Table~\ref{tab:differences-in-gradients}). These three metrics all suggest that the loss gradients of the float16 and float32 networks are very different, even though they have almost identical weights. Thus, the models diverge because the gradients are not pointing in the same direction at the beginning of training.

\begin{table}[htbp]
    \centering
    \caption{\textbf{Metrics to evaluate the differences in loss gradients between float16 and float32 at the first training iteration for the regression problem (Sec.~\ref{Function regression}).} The two networks of float16 and float32 have the same values at initialization. We used 10 random trials to calculate the mean and standard deviation.}
    \begin{tabular}{c c}
         \toprule
         
         Metric type & Value \\ 
         \midrule
         Cosine similarity & $0.16 \pm 0.20$\\
         $L^2$ distance & $\num{2.0440129935741425} \pm 1.0$\\
         $L^2$ relative distance & $\numpct{1.6698057115077973} \pm \numpct{1.1434329824813507}\%$\\
         \bottomrule
    \end{tabular}
    \label{tab:differences-in-gradients}
\end{table}

\begin{figure}[htbp]
    \centering
    \includegraphics[width=\textwidth]{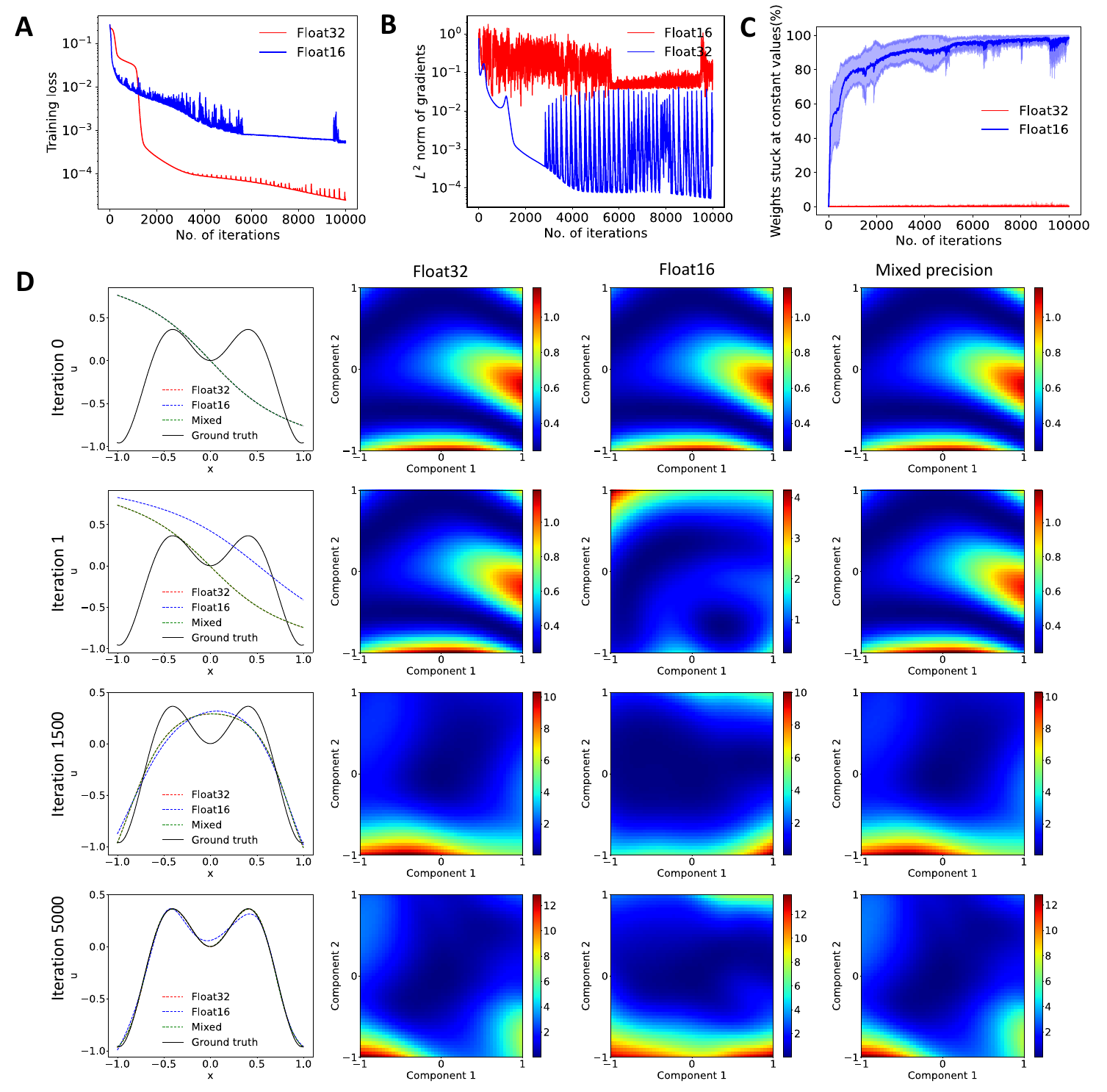}
    \caption{\textbf{Comparison of training the networks of float16, float32, and mixed precision.} Both networks are initialized with the same weights and biases.  (\textbf{A}) The training losses for float16 and float32 networks. (\textbf{B}) $L^2$ norm of the gradients of the training loss with respect to the network’s parameters. The loss gradients achieve much smaller values for float32 than float16. (\textbf{C}) The percentage of network weights that remain constant during training. The curves and shaded regions represent the mean and one standard deviation of $10$ runs. Almost all weights become constant for the float16 network. (\textbf{D}) Loss landscapes at different iterations between float32, float16, and mixed precision networks.}
    \label{fig:training-phases}
\end{figure}

\paragraph{Training loss, loss gradients, and network weights during network training.} We trained the float16 and float32 networks from the same network initialization. The float16 network underwent an initial period (the first $\sim$2000 iterations) in which it outperformed the float32 network in training loss (Fig.~\ref{fig:training-phases}A). However, the float16 network begins to learn very slowly after around 2000 iterations, resulting in a significantly higher training loss (Fig.~\ref{fig:training-phases}A). Comparing the loss gradients during network training provides more information on the different behaviors. In particular, the float16 network exhibits significantly larger gradients compared to float32 (Fig.~\ref{fig:training-phases}B). The higher initial gradient value of the float16 network leads to quick loss updates during the first $\sim$2000 iterations.
Furthermore, we found that a large portion of weights for the float16 network stays constant after some number of iterations (Fig.~\ref{fig:training-phases}C). Although the gradients are not zero (Fig.~\ref{fig:training-phases}B), the optimizer multiplies the gradients by the learning rate, and the final weight update underflows to zero  due to the limited precision of float16 and the rounding error~\cite{micikevicius_mixed_2018}. This disparity in the magnitudes of the gradients and stagnant weight updates contribute to the observed differences in training losses in Fig.~\ref{fig:training-phases}A.

\paragraph{Loss landscape during network training.} To further understand the source of the high optimization error during float16 network training, we analyzed the loss landscape of the float16 and float32 networks. The loss landscape is the representation of the loss function as a function of the network parameters. Analyzing the loss landscape for neural networks is important to understand how numerical precision affects the optimization process and the quality of the trained network. Using the visualization method in Ref.~\cite{visualloss}, we show high-dimensional loss landscapes in two dimensions at iterations 0, 1, 1500, and 5000 (Fig.~\ref{fig:training-phases}D). Specifically, we plot a 2D function $f(x,y) = \loss(\theta + x \delta + y \eta)$, where $\delta$ and $\eta$ are random vectors in the network weight space~\cite{visualloss}. 

We observe that for the float16 network, the loss landscape changes significantly from the 0th iteration to the 1st iteration compared to the float32 network, which barely changes. As network training progresses, the float16 loss landscape completely diverges compared to the float32 network. This is further supported by the magnitude of the weight gradients (Fig.~\ref{fig:training-phases}B).

\paragraph{Two training phases for understanding the failure of float16.} The analysis above in this section emphasizes a two-phase training process for the float16 network. 
\begin{itemize}
\item In the first phase, the loss gradients are large (Fig.~\ref{fig:training-phases}B), and therefore the float16 network learns faster and the training loss is below the float32 network (Fig.~\ref{fig:training-phases}A). 
\item In the second phase, the float16 network weights barely updates (Fig.~\ref{fig:training-phases}C) and the training loss remains relatively constant.
\end {itemize}

\subsection{PINNs} 
\label{heatequation}
In addition to a function regression problem, we also tested a PINN using float16. We demonstrate the low accuracy observed and an analysis of the loss landscape. 
\subsubsection{Heat equation} \label{prob:heatequation-setup}
We solve the heat equation 
\begin{equation*}
\frac{\partial u}{\partial t} = \alpha \frac{\partial^2 u}{\partial x^2}, \qquad x \in [0, 1], t \in [0, 1] 
\end{equation*} 
with the boundary condition
$ u(0,t) = u(1,t) = 0 $
and initial condition
$ u(x,0) = \sin\left (\frac{\pi x}{L} \right).$
Here, $\alpha = 0.4$ represents the thermal diffusivity constant.
Similar to the function regression problem, the float32 network outperforms the float16 network (Table~\ref{tab:regressionress}).

\subsubsection{Training difficulty and large optimization error for float16 with PINNs}
In this section, we provide similar analyses of the two networks just as
we did in function regression (Sec.~\ref{Function regression}). The PINN example follows the same trend as the function regression problem; however, the loss landscape of the PINN model further emphasizes the struggles of float16 during training.

We first confirm that both the networks of float32 and float16 can approximate the target solution well by casting the float32 network to the float16 network (Fig.~\ref{fig:combined_loss_landscapes}A). 
To visualize how training evolves, we compared the loss landscapes of the first two iterations between float16 and float32.  
A smooth loss landscape is desirable because it indicates that the optimization process is stable and can converge to a good solution. We illustrate that the loss landscape of float32 remains smooth at iteration 1, with a small change between iteration 0 and iteration 1 (Fig.~\ref{fig:combined_loss_landscapes}B). The loss landscape for float16 at iteration 0, on the other hand, exhibits roughness and diverges from the ideal landscape (Fig.~\ref{fig:combined_loss_landscapes}B). The empty white space represents numerical values that exploded to NaN due to numerical instability. By iteration 1, the instability intensifies, leading to more NaN values, and eventually does not converge to a good solution.

\begin{figure}[htbp]
    \centering
    \includegraphics[scale=.7 ]
    {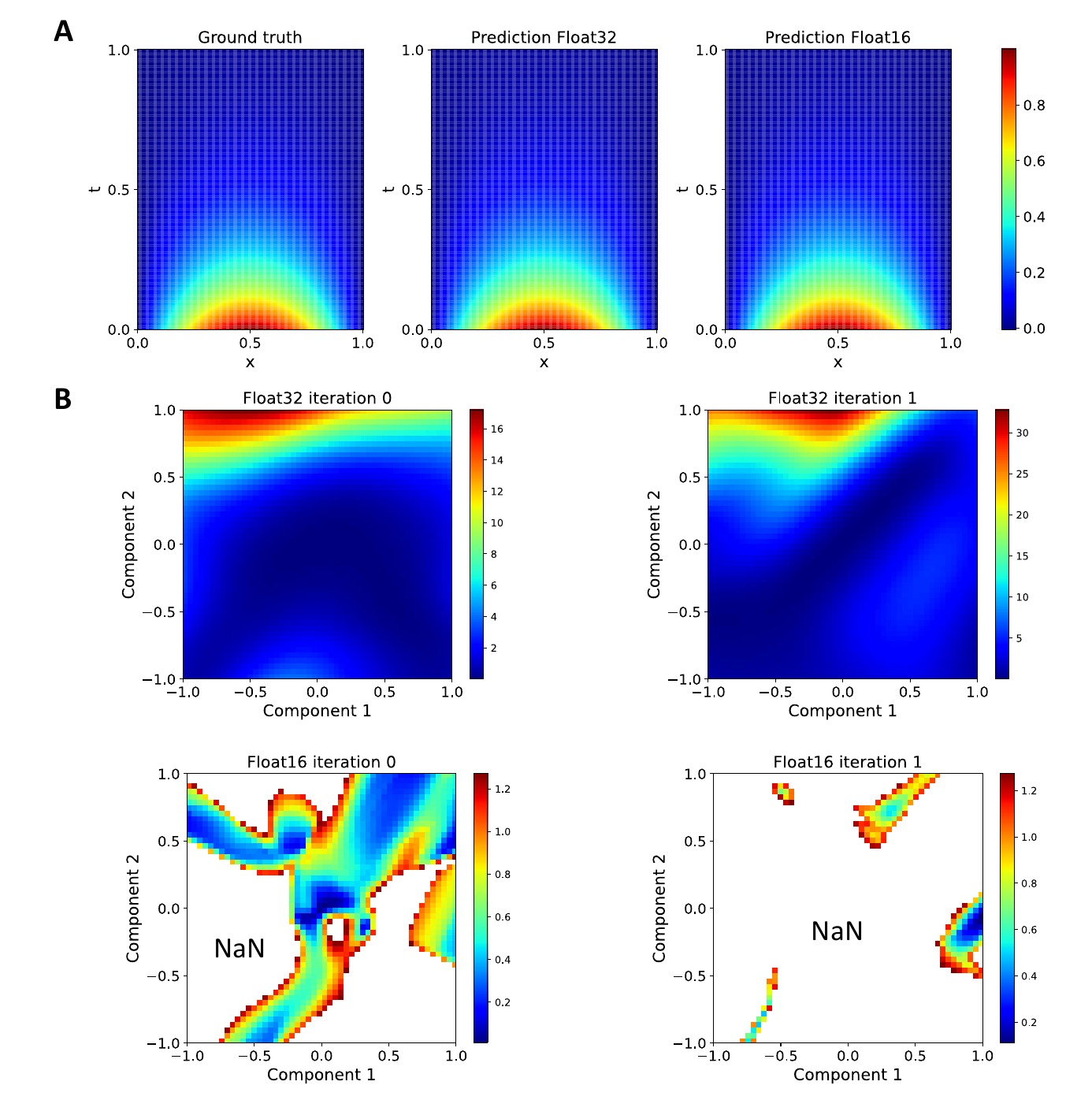}
    \caption{\textbf{Comparison of PINN predictions and loss landscapes at different iterations for the heat equation (Sec.~\ref{heatequation}).} (\textbf{A}) Train a PINN using float32, and then cast the network weights and biases from float32 to float16 after training. (\textbf{B}) The local loss landscapes of two networks at different iterations. The change in loss landscape for float32 network from iteration 0 to iteration 1 is smooth, while the change in loss landscape for float16 network from iteration 0 to iteration 1 is not smooth and the loss landscape has regions of NaNs.}
    \label{fig:combined_loss_landscapes}
\end{figure}

\section{Scientific machine learning with mixed precision: Methods} \label{sec:mixed-precision}

In this section, we explain the mixed precision method and how we applied it to SciML. 

\subsection{Mixed precision methods}

To enhance the efficiency and effectiveness of training deep neural networks, the utilization of mixed precision techniques has emerged. Mixed precision is an approach that combines the half precision (float16) presented in Sec.~\ref{sec:float16} and single precision (float32) numerical format to improve memory efficiency and computational speed.

\subsubsection{Generic mixed precision strategies}

There are two main strategies to deal with the inaccuracy of low-precision representations: mixed precision and loss scaling~\cite{nvidia, micikevicius_mixed_2018}. During training, mixed precision approaches the performance of the float32 network significantly better than that of the float16 network (Fig.~\ref{fig:training-phases}).

With a mixed precision network, a copy of float32 weights is maintained, but calculations, including forward and backward passes, are performed with float16 (Fig.~\ref{fig:flowchart}). The network computes the loss and gradients with float16 format, during which the bulk of the memory consumption and the calculations arise. The gradients are then converted back to the more precise float32 copy.

The second strategy, known as loss scaling, serves to maximize the utilization of the float16 data range when computing gradients (Fig.~\ref{fig:flowchart}). In the float16 representation, values can be accurately expressed within the range of $[6.10\cdot 10^{-5}, 6.55\cdot 10^{4}]$, but the actual values activated during training are usually less than one~\cite{nvidia}. When the gradients are scaled to larger values by multiplying with a scale factor, the range of activated values effectively shifts upward, reducing the risk of underflow in the backward pass. Note that the gradients are unscaled prior to their use in the stochastic gradient descent (SGD) algorithm. 

\begin{figure}[htbp]
\centering
\includegraphics[]{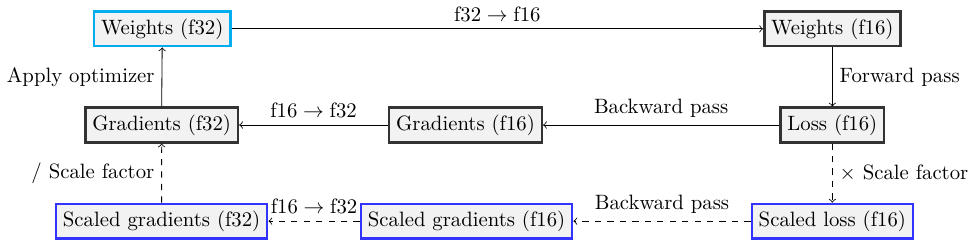}
\caption{\textbf{Flowchart of training networks with mixed precision.} The dashed path represents the optional loss scaling technique with a scale factor. Float32 and float16 are abbreviated as f32 and f16, respectively.}
\label{fig:flowchart}
\end{figure}

To implement mixed precision, for PyTorch, we used the PyTorch Automatic Mixed Precision (AMP). The core characteristic in the PyTorch AMP is the autocast feature, which intelligently manages the precision of various operations during network training. Under autocast, most calculations were performed in float16 to optimize the speed and memory efficiency. However, to maintain stability and accuracy, certain operations deemed prone to instability are performed in float32~\cite{amp}. If PyTorch AMP is used, any network evaluations done, e.g., mid-training testing, must be wrapped with autocast to avoid excess memory usage. For TensorFlow, we used the built-in support for mixed precision layer policies~\cite{tf-mixed-precision} by setting the global policy and the policy of a layer to \texttt{mixed\_float16}, which will cause the layer to store values in float32 and perform computations in float16.

\subsubsection{Additional techniques for SciML}

To enable mixed precision to better work with SciML, we introduce a number of additional techniques to mixed precision. We list them below.

\begin{enumerate}
    \item During the network training, we increased the value of the $\epsilon$ parameter in the Adam optimizer~\cite{kingma2017adam} from the default value of $10^{-7}$ to a larger value around $10^{-5}$. This adjustment was necessary to maintain numerical stability during the training process. 
    \item  For most DeepONet problems, we changed the loss function from the mean squared error to the mean $L^2$ relative error. We made this change because the mean squared error was overflowing when using mixed precision.
    \item In some cases, when the loss function was of the form $ A^2 / B $, i.e., the squared error $A^2$ normalized by a factor $B$, it overflowed. We fixed this by transforming it into the mathematically equivalent but less likely to overflow $ (A/\sqrt{B})^2$.
    \item In PINNs, when determining if a point is on a boundary/initial condition, we increased the absolute tolerance for the closeness function (\texttt{np.isclose}) from $10^{-6}$ to $10^{-4}$ due to the reduced precision of float16. The \texttt{np.isclose} function measures if two floating-point numbers are functionally the same (close), since floating point numbers are usually not exactly equal due to their representation in memory.
\end{enumerate}

Additionally, although others have found that loss scaling is useful~\cite{micikevicius_mixed_2018}, we did not find it useful for our work. Instead, we found that the only time an overflow occurred was in the loss function, and that the way to fix it was simply to change the loss function using the aforementioned techniques. In Sec.~\ref{sec:mixed-precision: Results}, all values of training time and memory usage are based on a TensorFlow implementation of mixed precision. During the preparation of this paper, a new paper~\cite{white2023speeding} appeared, which proposed to apply mixed precision to Fourier neural operators.

\section{Scientific machine learning with mixed precision: Results}  \label{sec:mixed-precision: Results}
We apply mixed precision techniques to PINNs, DeepONets, and physics-informed DeepONets.
To evaluate the performance of the networks, we compute the $L^2 $ relative error, and for each case, we record the training time and peak memory consumed and compare them. The hyperparameters of all the problems solved are in Appendix~\ref{sec:hyperparameters}.
\subsection{Mixed precision in PINNs}
In this section, we demonstrate the effectiveness of mixed precision in forward and inverse PINN problems. 
\subsubsection{Burgers' equation}\label{prob:burgers}

We solved the Burgers' equation below:  
$$
\frac{\partial u}{\partial t} + u\frac{\partial u}{\partial x} = \nu\frac{\partial^2u}{\partial x^2}, \qquad x \in [-1, 1], \quad t \in [0, 1],
$$
with the following boundary conditions,
$$
u(-1,t)=u(1,t)=0, \quad u(x,0) = - \sin(\pi x).
$$
The float16 network has a large $L^2$ relative error, and there is no major discrepancy in $L^2$ relative error between the float32 baseline network and the mixed precision network (Table ~\ref{tab:burgers-mixed}). The memory required for float16 and mixed precision networks were approximately half of the memory required for the float32 network. However, because we used a small network with a small number of training points for this problem (Table~\ref{tab:hyper-fnn}), we did not discover any time benefits.
\begin{table}[htbp]
\centering
\caption{\textbf{Comparison of the  $L^2$ relative error, time, and memory among float32, float16, and mixed precision networks for the Burgers' equation (Sec.~\ref{prob:burgers}). } }
\begin{tabular}{c|ccc}
\toprule
 & $L^2$ relative error & Time (s) & Memory (MB) \\
\midrule
Float32 & $ 3.21\pm 0.51 \%$ & $ 14.21\pm 0.20 $ &  $  8.63 \pm 0.07 $  \\
Float16 & $ 19.8\pm 1.5 \%$ & $ 15.49\pm 0.07 $ &  $  4.42\pm 0.02 $  \\
Mixed precision& $ 3.04\pm 0.22 \%$ & $ 16.09\pm 0.78 $  &  $  4.45\pm 0.03 $  \\

\bottomrule\end{tabular}
\label{tab:burgers-mixed}
\end{table}

\subsubsection{Kovasznay flow}\label{prob:kovasznay}
We solved the 2-dimensional Kovasznay flow on $\Omega = [0, 1]^2$,
\begin{align*}
\quad (\mathbf{u} \cdot \nabla) \mathbf{u} + \nabla p &= \frac{1}{\text{Re}} \nabla^2 \mathbf{u},
\end{align*}
with the following boundary conditions on $\partial \Omega$:
\[ u(x, y) = 1 - \exp(lx) \cos(2\pi y), \] 
\[ v(x, y) = \frac{l}{2\pi} \exp(lx) \sin(2\pi y). \]
For the pressure $p(x, y)$, a Dirichlet boundary condition is imposed on the right boundary with an outflow condition:
\[ p(x, y) = \frac{1}{2} \left(1 - \exp(2lx)\right) \quad \text{for} \quad x = 1, \] where $l$ is defined as $l = \frac{1}{2\nu} - \sqrt{\frac{1}{4\nu^2} + 4\pi^2}$ and Re = 20 is the Reynold's number used for this problem. 

The $L^2$ relative error for float16 is large for all variables, and the $L^2$ relative errors for mixed precision and float32 are lower and very similar (Table~\ref{tab:kovasznay} and Fig.~\ref{fig:kovasznay}). The mixed precision model uses around half as much memory as float32, and there is a $1.13\text{x}$ ($\approx245.70/216.74$) speedup when using mixed precision over float32. This example further validates our claim that mixed precision is a viable alternative to float32.

\begin{table}[h]
    \centering
    \caption{\textbf{Comparison of the $L^2$ relative error, time, and memory among float32, float16, and mixed precision networks for the Kovasznay flow problem (Sec.~\ref{prob:kovasznay}).}}
    \resizebox{\textwidth}{!}{
    \begin{tabular}{c|ccccc}
        \toprule
        \multirow{2}{*}{} & \multicolumn{3}{c}{$L^2$ relative error} & Time (s) & Memory (GB) \\
        \cmidrule{2-4}
         & $u$ & $v$ & $p$ & & \\
        \midrule
        Float32 & $0.40\pm0.15\%$ & $1.20\pm0.78\%$ & $0.50\pm0.13\%$ & $245.70 \pm 3.44$ & $1.51 \pm 0.02$\\
        Float16 & $7.50\pm0.41\%$ & $34.70\pm24.37\%$ & $23.10\pm1.24\%$ & $213.71 \pm 1.79$ & $0.76 \pm 0.01$\\
        Mixed precision & $0.60\pm0.04\%$ & $1.20\pm0.24\%$ & $0.60\pm0.09\%$ & $216.74 \pm 4.19$ & $0.76 \pm 0.01$\\
        \bottomrule
    \end{tabular}}
    \label{tab:kovasznay}
\end{table}

\begin{figure}[htbp]
    \centering
    \includegraphics[width=0.99\textwidth]{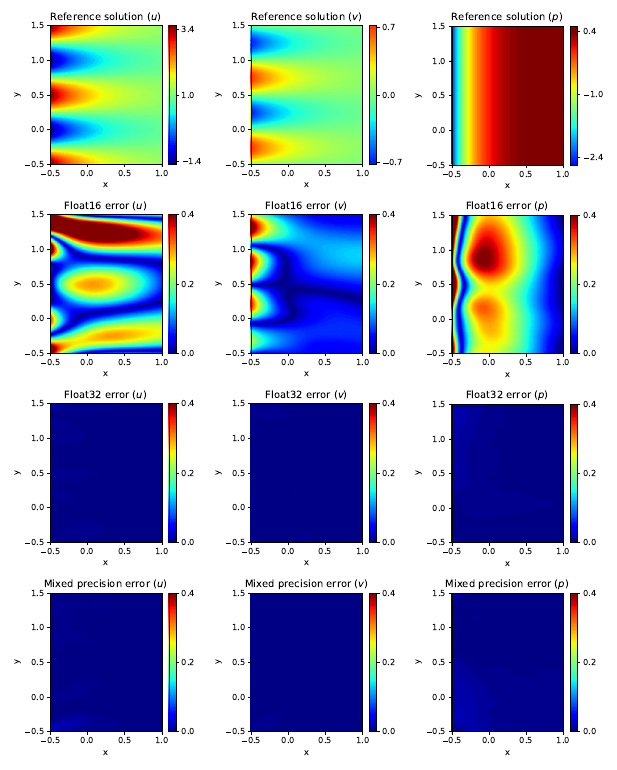}
    \caption{\textbf{Reference solutions and point-wise absolute error of velocity for the Kovasznay flow problem.} (First row) Reference solution. (Second row) The network prediction of float16 has large error. (Third and fourth rows) The network predictions of float32 and mixed precision have low error.}
    \label{fig:kovasznay}
\end{figure}

\subsubsection{Inverse problem of hemodynamics}
\label{prob:hemodynamics}
To demonstrate the potential of mixed precision training to significantly reduce training time, we solved a more computationally intensive problem. We simulate blood flow in the blood vessel with PINN, particularly focusing on platelet dynamics and thrombus formation~\cite{yazdani2018data}. The flow is constrained to satisfy the Navier-Stokes equations on the domain shown in Fig.~\ref{fig:hemodyanamics domain}A, 
\begin{align*}
\text{Momentum equation:} \quad (\mathbf{u} \cdot \nabla) \mathbf{u} + \nabla p &= \frac{1}{\text{Re}} \nabla^2 \mathbf{u},\\
\text{Continuity equation:} \quad \nabla \cdot \mathbf{u} &= 0,
\end{align*}
with no-slip boundary conditions applied on all solid faces.
This problem is an inverse problem, and we assume that the values for the inlet flow and the outlet pressure are unknown.

In terms of accuracy, the mixed precision network performed comparably to the standard float32 network (Table~\ref{tab:combined-hemo} and Fig.~\ref{fig:hemodyanamics domain}B). A significant reduction of 53\% in training time and 50\% in memory was observed, suggesting that, for computationally intensive problems, training with mixed precision allows for improved performance without sacrificing accuracy (Table~\ref{tab:combined-hemo}).
\begin{figure}[htbp]
    \centering
    \includegraphics[width=\textwidth]{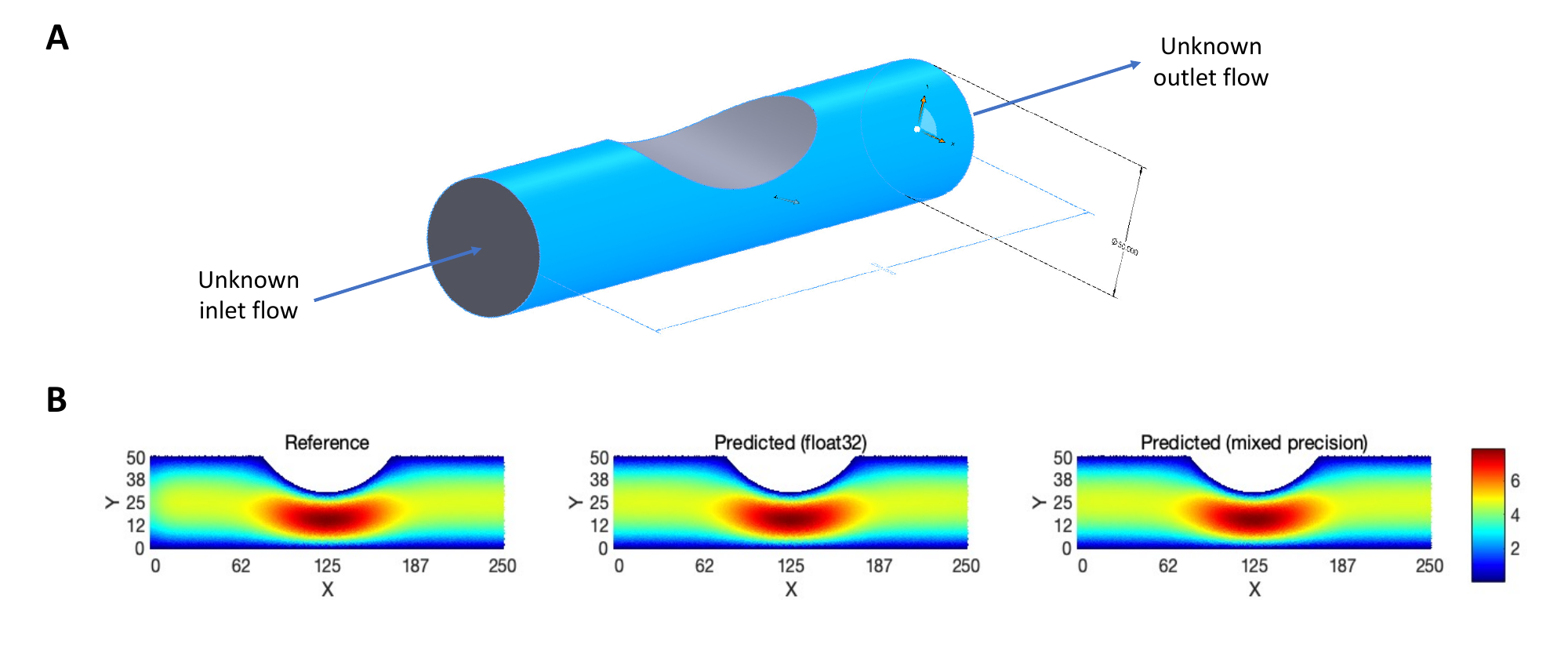}
    \caption{\textbf{Setup and PINN results for the hemodynamic problem (Sec.~\ref{prob:hemodynamics}).} (\textbf{A}) Computational domain. The inlet and outlet flow conditions are unknown. (\textbf{B}) The comparison of the magnitude of velocity for float32 and mixed precision network predictions at the center of the domain. The predicted solutions for float32 and mixed precision networks closely match the reference solution.}
    \label{fig:hemodyanamics domain}
\end{figure}
\begin{table}[htbp]
    \centering
    \caption{\textbf{Comparison of the $L^2$ relative error, time, and memory among float32, float16, and mixed precision networks for the hemodynamics problem (Sec.~\ref{prob:hemodynamics}).} }
    \begin{tabular}{c|ccc}
         \toprule
         & $L^2$ relative error of $||\textbf{u}||$ & Time (hours) & Memory (GB) \\
         \midrule
         Float32 & $2.23\pm0.02\%$ & $10.92 \pm 0.01$ & $11.25 \pm 0.01$\\
         Mixed precision & $2.30\pm0.01\%$ & $5.76 \pm 0.01$ & $5.61 \pm 0.01$ \\
         \midrule
         Ratio & 1.03 & 0.53 & 0.50 \\
         \bottomrule
    \end{tabular}
    \label{tab:combined-hemo}
\end{table}

\subsubsection{Inverse problem of flow in a rectangular domain}\label{prob:inverse}
To demonstrate the efficacy of mixed precision techniques in inverse problems, we tackled another inverse Navier-Stokes problem by focusing on incompressible flow in a rectangular domain $\Omega = [1, 8] \times [-2, 2]$ for a time period $[0, 7]$. The equation below describes the problem: 
\begin{align*}
\text{Momentum equation:} \quad \frac{\partial \mathbf{u}}{\partial t} + \lambda_1(\mathbf{u} \cdot \nabla) \mathbf{u} &= -\nabla p + \lambda_2 \nabla^2 \mathbf{u},\\
\text{Continuity equation:} \quad \nabla \cdot \mathbf{u} &= 0,
\end{align*}
where $\mathbf{u}$ is the velocity field and $p$ is the pressure field.
The true values of $\lambda_1$ and $\lambda_2$ are $1$ and $0.01$, respectively, and they denote the density and viscosity in the Navier-Stokes equation. Here, we assume that $\lambda_1$ and $\lambda_2$ are unknown, and the goal is to infer the values of $\lambda_1$ and $\lambda_2$ from the PDE solution data, i.e., the pressure field, the horizontal velocity, and the vertical velocity.  

The $L^2$ relative errors for velocity components $u$ and $v$ and the pressure field $p$ are lowest for the float32 network, indicating the highest accuracy (Table~\ref{tab:Inverse Navier Stokes}). The mixed precision network shows slightly higher but comparable $L^2$ relative errors, while the float16 network has significantly higher $L^2$ relative errors. The float32 network shows the lowest error when predicting $\lambda_1$.  Interestingly, the mixed precision network performs slightly better than float32 when predicting $\lambda_2$, although the error margin is quite large. This trend suggests that the mixed precision network can maintain a high level of accuracy, which is on par with the float32 network. The float16 network has the highest error, which is consistent with the trends observed in other variables (Table~\ref{tab:Inverse Navier Stokes}). There is a $1.75\text{x}$ ($\approx616.03/351.57$) speedup when using mixed precision over float32 while maintaining a 50\% ($\approx(1028.12-513.49)/1028.12$) memory reduction consistent with the float16 network.

\begin{table}[htbp]
\centering
\caption{\textbf{Comparison among float32, mixed, and float16 precision networks for the Navier-Stokes inverse problem  (Sec~\ref{prob:inverse}).}}
\begin{tabular}{cc|ccc}
\toprule
 & & Float32 & Mixed precision &  Float16 \\
\midrule
\multirow{3}{*}{$L^2$ relative error} & $u$ & $0.79\pm0.04\%$ &  $0.90\pm0.05\%$ & $3.39\pm0.25\%$\\
& $v$ & $2.57\pm0.05\%$ &  $3.33\pm0.14\%$ & $14.80\pm1.42\%$\\
& $p$ & $2.00\pm0.07\%$ &  $3.30\pm0.34\%$ & $12.70\pm0.62\%$\\
\midrule
\multirow{2}{*}{Relative error}& $\lambda_1$ & $0.03\pm0.06\%$ & $1.10\pm0.10\%$ & $3.70\pm0.46\%$ \\
& $\lambda_2$ & $7.00\pm1.73\%$ & $5.00\pm3.61\%$ & $15.70\pm7.02\%$ \\
\midrule
Training time (s) & & \makecell[c]{$616.03\pm 0.15$ } & \makecell[c]{$351.57\pm 0.70$} & \makecell[c]{$349.54\pm 0.31$} \\
Memory (MB) & & \makecell[c]{$1028.12 \pm 1.03$} & \makecell[c]{$513.49\pm 0.03$} & \makecell[c]{$513.68 \pm 0.33$} \\
\bottomrule
\end{tabular}
\label{tab:Inverse Navier Stokes}

\end{table}

\subsection{Mixed precision in DeepONet}
We apply mixed precision to DeepONet problems.

\subsubsection{Advection equation}\label{prob:advection}

We consider the advection equation with a periodic boundary condition~\cite{FAIR-comparison}:
$$
    \pder{u}{t} + \pder{u}{x} = 0, \quad \Omega = [0, 1]^2.
$$
We choose the initial condition as a square wave centered at $x$ of width $w$ and height $h$:
\[
u_0(x) = h \cdot \mathbf{1}_{[c - w/2, c + w/2]},
\]
where $(c, w, h)$ are randomly chosen from $[0.3, 0.7] \times [0.3, 0.6] \times [1, 2]$. We learn the mapping from the initial condition $u_0(x)$ to the solution at $t = 0.5$.

We observe notable findings in the relative $L^2$ error, the use of memory and the computation time (Table~\ref{tab:advection}). The prediction of mixed precision network demonstrates a comparable relative $L^2$ error of approximately 0.33\%, with a slightly wider range of variation (±0.04\%) compared to float32 network. The mixed precision network shows a noticeable reduction in memory usage, averaging approximately 34.63 MB (±0.0243 MB). Although it does not achieve an exact 50\% reduction in memory compared to float32, this reduction is still substantial. There is also a $1.11\text{x}$ ($\approx635.08/567.97$) speed-up when using mixed precision over float32. This result shows how well the mixed precision framework works with the DeepONet architecture.
\begin{table}[!ht]
\centering
\caption{\textbf{Comparison of the $L^2$ relative error, time, and memory among float32, float16, and mixed precision networks  for the advection problem (Sec.~\ref{prob:advection}).}}
\begin{tabular}{c|ccc}
\toprule
& $L^2$ relative error & Training time (s)  & Memory (MB)\\
\midrule
Float32 & $0.33\pm0.01\%$ & $635.08 \pm 3.85$ & $53.47 \pm 0.87$\\
Float16 & $0.51\pm0.02\%$ & $511.33 \pm 31.75$ & $31.46 \pm 0.01$\\
Mixed precision & $0.33\pm0.04\%$ & $567.97 \pm 3.68$ & $34.63 \pm 0.02$\\
\bottomrule
\end{tabular}
\label{tab:advection}
\end{table}

\subsubsection{Linear instability waves in high-speed boundary layers}\label{prob:pod-deeponet}

We applied the mixed precision method to the linear instability wave problem~\cite{di2021deeponet, FAIR-comparison}. The aim is to predict the evolution of linear instability waves in a compressible boundary layer. This involves studying how the initial upstream instability wave behaves as it travels downstream within a specific region. The focus is on instability waves with small amplitudes, which can be effectively characterized using linear parabolized stability equations. These equations are derived from the Navier–Stokes equations by dividing the flow into a base flow and a perturbation.
Here we consider air with Prandtl number = 0.72 and ratio of specific heats $\gamma = 1.4$. The free-stream Mach number is $\text{Ma} = 4.5$ and the free-stream temperature is $T_0 =  65.15\text{K}$. We set the inflow location of our configuration slightly upstream at $\sqrt{\text{Re}_{x0}} = 1800$. When the perturbation frequency is $\omega$, the instability frequency is $\omega10^6/\sqrt{\text{Re}_{x0}}$. More details about this problem and the dataset can be found in Ref.~\cite{di2021deeponet}. The output functions in the dataset differ by more than two orders of magnitude, so we first normalized all functions so that the maximum value of each function is 1. We then computed the POD modes and use this POD-DeepONet to solve this problem.

The float16 network fails in solving this problem, achieving an $L^2$ relative error greater than 100\%. The $L^2$ relative errors for mixed precision and float32 are lower at $22.9\%$ and $23.5\%$, respectively (Table~\ref{tab:liw-pod-deeponet}). We get a significant speedup of $1.48\text{x}$ ($\approx13.49/9.10$)  and a 50\% ($\approx(941.54-471.81)/941.54$) reduction in memory when training with mixed precision rather than float32, while still maintaining comparable accuracy (Table~\ref{tab:liw-pod-deeponet}). 
\begin{table}[htbp]
\centering
\caption{\textbf{Comparison of the $L^2$ relative error, time, and memory among float32, float16, and mixed precision networks for the linear instability wave (Sec.~\ref{prob:pod-deeponet}).}}
\begin{tabular}{c|ccc}
\toprule
& $L^2$ relative error & Training time (hours)  & Memory (MB)\\
 \midrule
Float32 & $22.9\pm1.30\%$ & $13.49 \pm 0.44$ & $941.54 \pm 0.02$\\
Float16 & $102.1\pm0.20\%$ & $8.43 \pm 0.23$ &  $471.81 \pm 0.01$\\
Mixed precision & $23.5\pm0.70\%$ & $9.10 \pm 0.05$ & $471.71 \pm 0.01$\\
\bottomrule
\end{tabular}
\label{tab:liw-pod-deeponet}
\end{table}

\subsubsection{Physics-informed DeepONet for solving the diffusion reaction equation}\label{sec:pi-deeponet}\label{prob:pi-deeponet}

In addition to data-driven DeepONets, we also tested a physics-informed DeepONet. We learn an operator for a diffusion reaction system given an arbitrary source term $v(x)$, satisfying 
\[
\pder{u}{t} = D \pdd{u}{x} - ku^2 + v(x), \quad \Omega = [0,1]^2,
\]
with $D = 0.01$ and $k = 0.01$. The boundary conditions are:
\[
u(x,0) = 0, \quad u(0,t) = u(1,t) = 0.
\]
A DeepONet was trained with a physics-informed loss function.

The mixed precision network achieves a $L^2$ relative error ($2.2\pm1.1\%$) comparable to the float32 network, indicating that it maintains high accuracy while benefiting from reduced precision in computations.
The float16 network exhibits a significantly higher $L^2$ relative error ($6.3\pm3.1\%$), which is almost three times higher than the float32 network. The mixed precision network achieved almost the same $L^2$ relative error as the float32 model, while saving $50\%$ ($\approx(269.84-135.80)/269.84$) on memory and $43\%$ ($\approx(1198-682)/1198$) on time (Table~$\ref{tab:pi-deeponet}$). This example further emphasizes the advantage that mixed precision has with physics-informed models, even when combined with a DeepONet architecture.
\begin{table}[htbp]
\centering
\caption{\textbf{Comparison of the $L^2$ relative error, time, and memory among float32, float16, and mixed precision networks  for the diffusion reaction problem (Sec.~\ref{prob:pi-deeponet}).}}
\begin{tabular}{c|ccc}
\toprule
& $L^2$ relative error & Training time (s)  & Memory (MB)\\ 
\midrule
Float32 & \(\numpct{0.022} \pm 1.0\%\) & \(1198\)  & \(269.84\)\\ 
Float16 & \(\numpct{0.063} \pm \numpct{0.031}\%\) & \( 678\)  & \(134.92\)\\ 
Mixed precision & \(\numpct{0.022} \pm \numpct{0.011}\%\) & \(682\)  & \(135.80\)\\ 
\bottomrule
\end{tabular}
\label{tab:pi-deeponet}
\end{table}

\section{Analysis of accuracy for mixed precision methods} \label{sec:theory}
In this section, we provide a theoretical basis for gradient descent with mixed precision in neural networks that require high precision such as PINNs and DeepONets. We analyze the behavior of gradient descent near a local minimum in the loss function.

\subsection{Theoretical analysis of the gradient and value of training loss}

We model the mixed precision gradient descent as follows. Suppose that $\theta$ is the vector containing the weights and biases, $\loss(\theta)$ is the training loss function, and $\eta$ is the learning rate. In a mixed precision network, the gradients are calculated with float16 and thus differ from the true gradients $\nabla \loss(\theta)$, and the weights are updated via
\[\theta \leftarrow \theta - \eta \nabla \loss(\theta + \delta \theta).\]
Here, $\delta \theta$ is the rounding error of float16 that satisfies $\lVert \delta \theta \rVert \le 2^{-11} \lVert \theta \rVert$.

In a neural network trained with forward/backward propagation in float16, inherent errors in the gradient prevent the network from reaching its minimum. In order to analyze the accuracy of the network, we consider that the loss function $\loss$ has a minimum at $\theta^*$ and assume the following:
\begin{enumerate}
    \item $\loss$ is convex.
    \item $\nabla \loss$ is Lipschitz continuous with a Lipschitz constant $L \le \frac{1}{\eta}$.
\end{enumerate}
Then we show that, by using gradient descent, the gradient of the loss function can be small enough (Theorem~\ref{thm:thm1}).
\begin{theorem} \label{thm:thm1}
The mixed precision loss function will reach some $\theta$ by gradient descent, such that 
\[\label{eq:thm1}
\lVert \nabla \loss(\theta) \rVert < \frac{2 + \sqrt{6}}{2^{11}}L \lVert \theta \rVert. \tag{6.1.1}\]
\end{theorem}
\begin{proof}
The proof can be found in Appendix~\ref{sec:appendix:proofs}.
\end{proof}

If we further assume strong convexity, we can use Theorem~\ref{thm:thm1} to derive the bound of the loss function.
\begin{corollary} \label{thm:corollary}
Suppose that $\loss$ is strongly convex with $\frac{\lVert\nabla \loss(x) - \nabla \loss(y) \rVert}{\lVert x - y \rVert} \ge \mu$ for any $x$ and $y$, and some $\mu > 0$. Then the mixed precision loss function will reach some $\theta$ by gradient descent, such that
\[\lVert \theta - \theta^*\rVert < \frac{2 + \sqrt{6}}{2^{11}}\cdot \frac{ L}{\mu} \lVert \theta \rVert\]
and
\[\loss(\theta) -\loss(\theta^*) < \frac{15 + 6\sqrt{6}}{2^{22}} \frac{L^2}{\mu} \lVert \theta \rVert^2.\]
\end{corollary}
\begin{proof}
The proof can be found in Appendix~\ref{sec:appendix:proofs_Corollary}.
\end{proof}

\subsection{Experimental validation}
\label{prob:validation}

We validated Theorem~\ref{thm:thm1} on the following diffusion equation in one dimension
\[\pder{y}{t} = \pdd{y}{x} - e^{-t} \left( \sin(\pi x) - \pi^2 \sin (\pi x) \right), \qquad x\in[-1, 1], \quad t\in [0, 1],\]
with initial and boundary conditions
\[y(x, 0) = \sin(\pi x), \qquad y(-1, t) = y(1, t) = 0.\]
In order to compute Eq.~\eqref{eq:thm1}, we needed to determine the Lipschitz constant $L$ of the loss function gradient. As the exact computation of this Lipschitz constant is infeasible, we obtained an estimate by computing $ \frac{\lVert\nabla \loss(\theta_{t+1}) - \nabla \loss(\theta_t) \rVert}{\lVert \theta_{t+1} - \theta_t \rVert}$ at each step $t$ of the training process. This estimated local Lipschitz constant of the loss function gradient ranged from around $10^0$ to $10^2$ (Fig.~\ref{fig:mixed-diffusion}A). Using the estimated Lipschitz constant, the model is indeed able to achieve a low gradient of the loss function as predicted by Theorem~\ref{thm:thm1} (Fig.~\ref{fig:mixed-diffusion}B).

\begin{figure}[htbp]
    \centering
    \includegraphics[width=\textwidth]{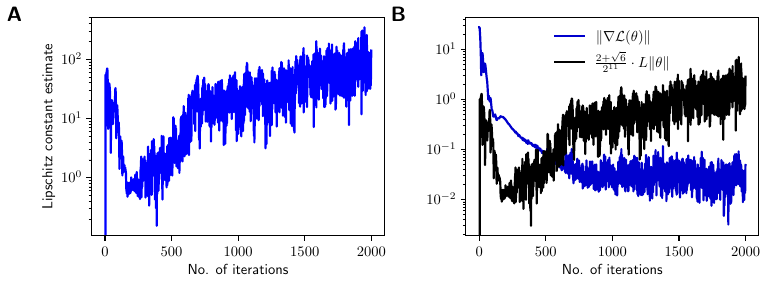}
    \caption{\textbf{Validation of Theorem~\ref{thm:thm1} for the diffusion equation.} (\textbf{A}) Estimate of the (local) Lipschitz constant of the loss gradient, which is used in the theoretical upper bounds in Eq.~\eqref{eq:thm1}. (\textbf{B}) The norm of the loss gradient $\lVert \nabla \loss(\theta) \rVert$ and $\frac{2 + \sqrt{6}}{2^{11}}L \lVert \theta \rVert$ correspond to the left- and right-hand sides of Eq.~\eqref{eq:thm1}.}
    \label{fig:mixed-diffusion}
\end{figure}

\section{Conclusions}\label{sec:conclusion}
We analyze the failures and training difficulties associated with using float16 for SciML. Despite the memory and computational speed benefits, the float16 network struggles with high optimization error during training, as observed in function regression and PINN problems. The divergence of gradients and the \text{weight updates going to zero} contribute substantially to the decreased accuracy of the float16 network. The analysis of the loss landscapes further emphasizes the instability and divergence issues inherent in float16 training. 

We demonstrate that using mixed precision, which combines both float16 and float32 numerical formats, improves memory efficiency and computational speed without sacrificing accuracy, when training PINNs and DeepONets. The key strategy in mixed precision is to maintain a float32 copy of weights while performing network forward and backward propagations in float16. We also performed extensive experiments with PINNs and DeepONets, and consistently demonstrated that mixed precision not only matches the accuracy of float32, but also offers considerable improvements in memory and computational speed. Finally, we show, with theoretical and empirical analysis, the rounding error due to float16.

There are still limitations to mixed precision training of PINNs and DeepONets. We do not always obtain the theoretically optimal halving of memory and the training times are not consistently sped up by the same amount. Additionally, although loss scaling can usually prevent overflow or underflow in gradients, it does not prevent overflow in the intermediate calculations inside loss functions. In this case, we manually change or modify the loss function. In the future, a tool that automatically detects and mitigates overflows, similar to how loss scaling works, would be useful.


\section*{Acknowledgements}

This work was supported by the U.S. Department of Energy [DE-SC0022953]. E.W. and L.L. thank MIT's PRIMES-USA program.

\appendix
\section{Hyperparamaters}
\label{sec:hyperparameters}

For all the function regression, PINN and DeepONet problems, a consistent set of hyperparameters were used, and are detailed in Tables~\ref{tab:hyper-fnn} and~\ref{tab:hyper-deeponets}. Across these problems, the Adam optimizer was used for optimization. The two tables outline the depth, width, activation function, learning rate, number of iterations, and training points for each problem. The GPU used for each problem is listed in Table~\ref{tab:gpus}.

For the hemodynamics problem (Sec.~\ref{prob:hemodynamics}), a specialized neural network architecture known as a parallel feedforward neural network (PFNN), which consists of multiple independent sub-networks, each with its own hidden layers and neurons, was implemented. In our implementation, the PFNN is comprised of two sub-networks: one with a depth of 3 and width of 128 and another sub-network with 4 hidden layers of width 64, 64, 64, and 16, respectively. The learning rate was kept at $0.001$ for the first 160000 iterations and lowered to $0.0001$ for the last 50000 iterations. 

\begin{table}[htbp]
    \centering
    \caption{\textbf{The hyperparameters used for the regression problems and PINNs.}}
    \small
    \begin{tabular}{l|cccccc}
         \toprule
         Problem & Depth & Width & Activation & Learning rate & Iterations & Training points\\
         \midrule
         Sec. \ref{Function regression} Regression problem & 3 & 10 & tanh & $0.001$ & $10000$ & $16$\\
         Sec. \ref{heatequation} Heat equation & 4 & 20 & tanh & $0.001$ & $20000$ & $2780$\\
         Sec. \ref{prob:burgers} Burgers' equation & 3 & 32 & tanh & $0.001$ & $20000$ & 2240 \\ 
         Sec. \ref{prob:hemodynamics} Hemodynamics & --- & --- & Swish & --- & $210000$ & 10000\\
         Sec. \ref{prob:kovasznay} Kovasznay flow & 5 & 100 & tanh & 0.001 & 30000 & $3000$\\
         Sec. \ref{prob:inverse} Navier-Stokes inverse & 6 & 60 & tanh & 0.001 & 20000 & $8000$\\
         Sec. \ref{prob:validation} Diffusion reaction equation & 4 & 32 & tanh & 0.001 & 2000 & 260 \\ 
         \bottomrule
    \end{tabular}
    \label{tab:hyper-fnn}
\end{table}

\begin{table}[ht]
    \centering
    \caption{\textbf{The hyperparameters used for the DeepONets.} The trunk and branch network sizes are expressed as depth $\times$ width.}
    \footnotesize
    \begin{tabular}{l|ccccc}
        \toprule
        Problem & Trunk size & Branch size & Activation & Learning rate & Iterations\\
        \midrule
         Sec. \ref{prob:pod-deeponet} Linear instability wave & $6 \times 256$ & 9 POD modes & ELU & 0.001 & 5000000\\
         Sec. \ref{prob:advection} Advection equation & $4 \times 512$ & $2 \times 512$ & ReLU & 0.001 & 250000\\
         Sec. \ref{prob:pi-deeponet} Diffusion reaction equation & $3 \times 128$ & $3 \times 128$ & tanh & 0.0005 & 50000\\
         \bottomrule
    \end{tabular}
    \label{tab:hyper-deeponets}
\end{table}

\begin{table}[H]
    \centering
    \caption{\textbf{The GPU used for each problem.}}\label{tab:gpus}
    \begin{tabular}{l|l}
        \toprule
        Problem & GPU \\
        \midrule
        Sec. \ref{Function regression} Regression problem & NVIDIA Tesla T4 \\
        Sec. \ref{heatequation} Heat equation & NVIDIA GeForce RTX 3090 \\
        Sec. \ref{prob:burgers} Burgers' equation & NVIDIA GeForce RTX 4090  \\
        Sec. \ref{prob:hemodynamics} Hemodynamics & NVIDIA GeForce RTX 3090 \\
        Sec. \ref{prob:kovasznay} Kovasznay flow & NVIDIA GeForce RTX 3090 \\
        Sec. \ref{prob:inverse} Navier-Stokes inverse & NVIDIA GeForce RTX 3090 \\
        Sec. \ref{prob:advection} Advection equation & NVIDIA GeForce RTX 3090 \\
        Sec. \ref{prob:pod-deeponet} Linear instability wave & NVIDIA GeForce RTX 3090 \\
        Sec. \ref{prob:pi-deeponet} Diffusion reaction equation & NVIDIA Tesla T4 \\
        Sec. \ref{prob:validation} Diffusion equation & NVIDIA Tesla T4\\
         \bottomrule
    \end{tabular}
    \label{tab:hyper-GPUs}
\end{table}

\section{Proof of Theorem~\ref{thm:thm1}}
\label{sec:appendix:proofs}

\begin{proof}
Since $\nabla \loss$ has Lipschitz constant $L$, we have $\frac{\lVert\nabla \loss(x) - \nabla \loss(y) \rVert}{\lVert x - y \rVert} \le L$ for all $x$ and $y$. We first give two bounds regarding the error of the gradients. Note that 
\begin{align*}
\nabla \loss(\theta)^T (\nabla \loss(\theta + \delta \theta)) &= \nabla \loss(\theta)^T\nabla \loss(\theta) + \nabla \loss(\theta)^T(\nabla \loss(\theta + \delta \theta) - \nabla \loss(\theta))\\
&\ge \lVert \nabla \loss(\theta)\rVert^2 - \lVert \nabla \loss(\theta)\rVert\lVert \nabla \loss(\theta + \delta \theta) - \nabla \loss(\theta)\rVert\\
&\ge \lVert \nabla \loss(\theta)\rVert^2 - L\lVert \nabla \loss(\theta) \rVert\lVert \delta \theta \rVert
\tag{B.1}\label{eq:dot-product-bound}
\end{align*} 
and
\begin{align*}
\lVert \nabla \loss(\theta + \delta \theta) \rVert^2 &\le (\lVert \nabla \loss(\theta) \rVert + L\lVert \delta \theta\rVert)^2 = \lVert \nabla \loss(\theta)\rVert^2 + 2L\lVert \nabla \loss(\theta)\rVert\lVert \delta \theta\rVert + L^2\lVert \delta \theta\rVert^2.
\tag{B.2}
\end{align*}
We now bound the updated loss after one iteration of gradient descent. The Lipschitz continuity gives us \cite{lipschitz-properties} that for all $y$,
\begin{align*}
\loss(y) \le \loss(\theta) + \nabla \loss(\theta)^T(y-\theta) + \frac 12 L \lVert y -\theta\rVert^2. \tag{B.3}
\end{align*}
Thus, taking $y = \theta' = \theta - \eta \nabla \loss(\theta + \delta \theta)$ and using Eqs. (B.1), (B.2) and (B.3), we have
\begin{align*}
\loss(\theta') &\le \loss(\theta) + \nabla \loss(\theta)^T(-\eta\nabla \loss(\theta + \delta \theta)) + \frac 12 L \lVert \eta \nabla \loss(\theta + \delta \theta)\rVert^2\\
&\le \loss(\theta) - \eta\lVert \nabla \loss(\theta)\rVert^2 + L\eta \lVert \nabla \loss(\theta) \rVert \lVert \delta \theta \rVert + \frac 12 L \eta ^2 (\lVert \nabla \loss(\theta)\rVert^2 + 2L\lVert \nabla \loss(\theta)\rVert\lVert \delta \theta\rVert + L^2\lVert \delta \theta\rVert^2)\\
&= \loss(\theta) - \left(\eta - \frac 12 L \eta ^2\right)\lVert \nabla \loss(\theta) \rVert^2 + \frac 12 L \eta ^2 (2L\lVert \nabla \loss(\theta)\rVert\lVert \delta \theta\rVert + L^2\lVert \delta \theta\rVert^2).
\end{align*}
From $L\eta \le 1$, we have
\begin{align*}
\loss(\theta') & \le \loss(\theta) - \left( \eta - \frac 12 \eta \right)\lVert \nabla \loss(\theta) \rVert^2 + \frac 12 \eta (2L\lVert \nabla \loss(\theta)\rVert\lVert \delta \theta\rVert + L^2\lVert \delta \theta\rVert^2)\\
&= \loss(\theta) - \frac 12 \eta \lVert \nabla \loss(\theta) \rVert^2 + \eta L\lVert \nabla \loss(\theta)\rVert\lVert \delta \theta \rVert + \frac 12 \eta L^2\lVert \delta \theta \rVert ^2\\
&= \loss(\theta) - \frac 12 \eta \lVert \nabla \loss(\theta) \rVert^2 + \eta\left( L\lVert \nabla \loss(\theta)\rVert\lVert \delta \theta \rVert + \frac 12  L^2\lVert \delta \theta \rVert ^2\right).\tag{B.4}
\end{align*}
If Eq. (\ref{eq:thm1}) holds, we will consider the model to be in a \textit{critical region}. We now show that the model always reaches the critical region, thereby proving Theorem~\ref{thm:thm1}. Suppose, for the sake of contradiction, that the model never reaches the critical region. We have
\begin{align*}
\lVert \nabla \loss(\theta) \rVert \ge \frac{2 + \sqrt{6}}{2^{11}}L \lVert \theta \rVert &\implies \lVert \nabla \loss(\theta) \rVert \ge (2 + \sqrt{6})L \lVert \delta \theta \rVert\\
&\implies \lVert \nabla \loss (\theta) \rVert^2 - 4L\lVert \delta \theta \rVert \lVert \nabla \loss (\theta) \rVert - 2L^2\lVert \delta \theta\rVert^2 \ge 0\\
&\implies L\lVert \delta \theta\rVert \lVert \nabla \loss(\theta)\rVert + \frac 12 L^2 \lVert \delta \theta\rVert^2 \le \frac 14 \lVert \nabla \loss(\theta) \rVert^2.
\end{align*}
From Eq. (B.4), we have
\[\loss(\theta') \le \loss(\theta) - \frac 14 \eta \lVert \nabla \loss(\theta)\rVert ^2. \tag{B.4}\]
Let $\theta_t$ be the model parameters at iteration $t$ and define $E_t = \loss(\theta_t) - \loss(\theta^*)$. By convexity, we have
\[\lVert \theta_t - \theta^* \rVert\lVert \nabla \loss (\theta_t) \rVert \ge \loss (\theta_t) - \loss (\theta^*) \implies \lVert \nabla \loss(\theta_t) \rVert \ge \frac{E_t}{\lVert \theta_t - \theta^* \rVert},\]
and Eq. (B.4) gives us
\[E_{t+1} \le E_t - \frac 14 \eta \lVert \nabla \loss(\theta_t)\rVert^2 \le E_t - \frac 14 \eta \frac{E_t^2}{\lVert \theta_t - \theta^*\rVert^2} \le E_t - \frac 14 \eta \frac{E_t^2}{\lVert \theta_1 - \theta^*\rVert^2}. \tag{B.6}\]
We now prove the following bound on the error with induction:
\[E_t \le \frac{\alpha}{t}, \quad \text{where~} \alpha = \max \left\{ \frac{8 \lVert \theta_1 - \theta^* \rVert^2}{\eta}, E_1\right\}.\]
The base case is
\[E_1 = \frac{E_1}{1} \le \frac{\alpha}{1}.\]
Now assume the inductive hypothesis holds for $t$. We have, by Eq. (B.6),
\[
E_{t+1} \le E_t - \frac 14 \eta \frac{E_t^2}{\lVert \theta_1 - \theta^*\rVert^2} \le  E_t - \frac{2}{\alpha} E_t^2.
\]
If $E_t \le \frac{\alpha}{t+1}$, then we are done, since $E_{t+1} \le E_t$. Otherwise, we have
\[
E_{t+1} \le \frac{\alpha}{t} - \frac{2}{\alpha} \frac{\alpha^2}{(t+1)^2} = \frac{\alpha}{t} - \frac{2\alpha}{(t+1)^2} = \alpha \left(\frac 1t - \frac 2{(t+1)^2}\right).
\]
Since 
\[
\frac 1t - \frac 2{(t+1)^2} = \frac{\frac{t+1}{t} - \frac{2}{t+1}}{t+1} = \frac{1 + \frac{1}{t} - \frac{2}{t+1}}{t+1} \le \frac{1}{t+1},
\]
the inductive step is proven. Hence, $E_t$ tends towards zero and the value of the loss function approaches its minimum. But since $\loss(\theta)$ is convex, this means that $\nabla \loss(\theta)$ goes to zero, and thus the assumption that the model never reaches the critical region was incorrect.
\end{proof}

\section{Proof of Corollary~\ref{thm:corollary}}
\label{sec:appendix:proofs_Corollary}

\begin{proof}
If $\loss$ is strongly convex with $\mu \le \frac{\lVert\nabla \loss(x) - \nabla \loss(y) \rVert}{\lVert x - y \rVert}$, then using Theorem~\ref{thm:thm1}, we have that inside the critical region,
\[\lVert \theta - \theta^*\rVert \le \frac{1}{\mu}\lVert \nabla \loss(\theta) \rVert \le \frac{2 + \sqrt{6}}{2^{11}}\cdot \frac{L}{\mu} \lVert \theta \rVert.\]
Hence, using Eq. (B.3),
\[\loss(\theta) \le \loss(\theta^*) + \lVert \nabla \loss(\theta) \rVert\lVert \theta - \theta^*\rVert + \frac 12 L \lVert \theta - \theta^*\rVert^2 \le \loss(\theta^*) + \frac{3L^2}{2\mu}\left(\frac{2+\sqrt{6}}{2^{11}}\right)^2 \lVert \theta \rVert^2,\]
and so
\[\loss(\theta) -\loss(\theta^*)\le \frac{15 + 6\sqrt{6}}{2^{22}} \frac{L^2}{\mu} \lVert \theta \rVert^2.\]
\end{proof}

\printbibliography

@inproceedings{visualloss,
  title={Visualizing the Loss Landscape of Neural Nets},
  author={Li, Hao and Xu, Zheng and Taylor, Gavin and Studer, Christoph and Goldstein, Tom},
  booktitle={Neural Information Processing Systems},
  year={2018}
}

@inproceedings{
micikevicius_mixed_2018,
title={Mixed Precision Training},
author={Paulius Micikevicius and Sharan Narang and Jonah Alben and Gregory Diamos and Erich Elsen and David Garcia and Boris Ginsburg and Michael Houston and Oleksii Kuchaiev and Ganesh Venkatesh and Hao Wu},
booktitle={International Conference on Learning Representations},
year={2018}
}

@ARTICLE{pinns,
       author = {{Raissi}, M. and {Perdikaris}, P. and {Karniadakis}, G.~E.},
        title = "{Physics-informed neural networks: A deep learning framework for solving forward and inverse problems involving nonlinear partial differential equations}",
      journal = {Journal of Computational Physics},
         year = 2019,
       volume = {378},
        pages = {686-707}
}

@article{materials,
author = {Shukla, Khemraj and Clark di Leoni, Patricio and Blackshire, James and Sparkman, Daniel and Karniadakis, George},
year = {2020},
pages = {},
title = {Physics-Informed Neural Network for Ultrasound Nondestructive Quantification of Surface Breaking Cracks},
volume = {39},
journal = {Journal of Nondestructive Evaluation}
}

@article{hard_constraints,
author = {Lu, Lu and Pestourie, Raphaël and Yao, Wenjie and Wang, Zhicheng and Verdugo, Francesc and Johnson, Steven},
year = {2021},
pages = {B1105-B1132},
title = {Physics-Informed Neural Networks with Hard Constraints for Inverse Design},
volume = {43},
journal = {SIAM Journal on Scientific Computing}
}

@article{heat_transfer,
    author = {Cai, Shengze and Wang, Zhicheng and Wang, Sifan and Perdikaris, Paris and Karniadakis, George Em},
    title = {Physics-Informed Neural Networks for Heat Transfer Problems},
    journal = {Journal of Heat Transfer},
    volume = {143},
    number = {6},
    year = {2021},
}

@article{fluids,
  title={Data-driven physics-informed constitutive metamodeling of complex fluids: A multifidelity neural network (MFNN) framework},
  author={Mohammadamin Mahmoudabadbozchelou and Marco Caggioni and Setareh Shahsavari and William Hartt and George Em Karniadakis and Safa Jamali},
  journal={Journal of Rheology},
  year={2021},
  volume={65},
  pages={179-198}
}

@ARTICLE{biomedicine,
  
AUTHOR={Sahli Costabal, Francisco and Yang, Yibo and Perdikaris, Paris and Hurtado, Daniel E. and Kuhl, Ellen},   
	 
TITLE={Physics-Informed Neural Networks for Cardiac Activation Mapping},      
	
JOURNAL={Frontiers in Physics},      
	
VOLUME={8},           
	
YEAR={2020}
}

@online{nvidia,
author = {Nvidia},
title = {Train With Mixed Precision},
url = {https://docs.nvidia.com/deeplearning/performance/mixed-precision-training/index.html},
update={2023}
}

@online{amp,
author = {PyTorch},
title = {AUTOMATIC MIXED PRECISION PACKAGE - TORCH.AMP},
url = {https://pytorch.org/docs/stable/amp.html},
update = {2023}
}

@online{tf-mixed-precision,
title = {Mixed precision {\textbar} {TensorFlow} {Core}},
url = {https://www.tensorflow.org/guide/mixed_precision},
language = {en},
journal = {TensorFlow},
}

@article{FAIR-comparison,
title = {A comprehensive and fair comparison of two neural operators (with practical extensions) based on FAIR data},
journal = {Computer Methods in Applied Mechanics and Engineering},
volume = {393},
pages = {114778},
year = {2022},
author = {Lu Lu and Xuhui Meng and Shengze Cai and Zhiping Mao and Somdatta Goswami and Zhongqiang Zhang and George Em Karniadakis}
}

@article{kingma2017adam,
author = {Kingma, Diederik and Ba, Jimmy},
year = {2014},
pages = {},
title = {Adam: A Method for Stochastic Optimization},
journal = {International Conference on Learning Representations}
}

@article{yun2023defense, 
      title={In Defense of Pure 16-bit Floating-Point Neural Networks}, 
      author={Juyoung Yun and Byungkon Kang and Francois Rameau and Zhoulai Fu},
      year={2023},
      journal = {arXiv preprint arXiv:2305.10947 [cs.LG]}
}

@article{lu_deeponet_2021,
	title = {Learning nonlinear operators via {DeepONet} based on the universal approximation theorem of operators},
	volume = {3},
	copyright = {2021 The Author(s), under exclusive licence to Springer Nature Limited},
	language = {en},
	number = {3},
	journal = {Nature Machine Intelligence},
	author = {Lu, Lu and Jin, Pengzhan and Pang, Guofei and Zhang, Zhongqiang and Karniadakis, George Em},
	year = {2021},
	note = {Number: 3
Publisher: Nature Publishing Group},
	pages = {218--229},
}

@article{white2023speeding,
      title={Speeding up Fourier Neural Operators via Mixed Precision}, 
      author={Colin White and Renbo Tu and Jean Kossaifi and Gennady Pekhimenko and Kamyar Azizzadenesheli and Anima Anandkumar},
      year={2023},
      journal={arXiv preprint arXiv 2307.15034 [cs.LG]}
}

@article{lipschitz-properties,
title = {On the Fenchel Duality between Strong Convexity and Lipschitz
Continuous Gradient},
author = {Xingyu Zhou},
year = {2018},
journal={arXiv preprint arXiv 1803.06573 [math.OC]}
}

@article{di2021deeponet,
title = {Neural operator prediction of linear instability waves in high-speed boundary layers},
journal = {Journal of Computational Physics},
volume = {474},
pages = {111793},
year = {2023},
author = {Patricio {Clark Di Leoni} and Lu Lu and Charles Meneveau and George Em Karniadakis and Tamer A. Zaki}
}

@article{yazdani2018data,
  author = {Alireza Yazdani and He Li and Matthew R. Bersi and Paolo Di Achille and Joseph Insley and Jay D. Humphrey and George Em Karniadakis},
  title = {Data-driven Modeling of Hemodynamics and its Role on Thrombus Size and Shape in Aortic Dissections},
  journal = {Scientific Reports},
  volume = {8},
  pages = {2515},
  year = {2018},
  accessdate = {2023-10-08},
}

@article{brunton2023machine,
Author = {Steven L. Brunton and J. Nathan Kutz},
Title = {Machine Learning for Partial Differential Equations},
Year = {2023},
journal = {arXiv preprint arXiv:2303.17078},
}

@inbook{gholami2022a,
author = {Gholami, Asghar and Kim, Sehoon and Zhen, Dong and Yao, Zhewei and Mahoney, Michael and Keutzer, Kurt},
year = {2022},
pages = {291-326},
title = {A Survey of Quantization Methods for Efficient Neural Network Inference},
publisher = {Chapman \& Hall}
}

@inproceedings{hinton2015distilling,
title	= {Distilling the Knowledge in a Neural Network},
author	= {Geoffrey Hinton and Oriol Vinyals and Jeffrey Dean},
year	= {2015},
booktitle	= {NIPS Deep Learning and Representation Learning Workshop}
}

@article{jia2018highly,
      title={Highly Scalable Deep Learning Training System with Mixed-Precision: Training ImageNet in Four Minutes}, 
      author={Xianyan Jia and Shutao Song and Wei He and Yangzihao Wang and Haidong Rong and Feihu Zhou and Liqiang Xie and Zhenyu Guo and Yuanzhou Yang and Liwei Yu and Tiegang Chen and Guangxiao Hu and Shaohuai Shi and Xiaowen Chu},
      year={2018},
      journal={arXiv preprint arXiv 1807.11205 [cs.LG]}
}

@inproceedings{hanlearning2015,
author = {Han, Song and Pool, Jeff and Tran, John and Dally, William J.},
title = {Learning Both Weights and Connections for Efficient Neural Networks},
year = {2015},
publisher = {MIT Press},
address = {Cambridge, MA, USA},
booktitle = {Proceedings of the 28th International Conference on Neural Information Processing Systems - Volume 1},
pages = {1135–1143},
numpages = {9},
location = {Montreal, Canada},
series = {NIPS'15}
}

@article{zhubayesian2018,
title = {Bayesian deep convolutional encoder–decoder networks for surrogate modeling and uncertainty quantification},
journal = {Journal of Computational Physics},
volume = {366},
pages = {415-447},
year = {2018},
author = {Yinhao Zhu and Nicholas Zabaras}
}

@article{zhu_physics-constrained_2019,
	title = {Physics-constrained deep learning for high-dimensional surrogate modeling and uncertainty quantification without labeled data},
	volume = {394},
	journal = {Journal of Computational Physics},
	author = {Zhu, Yinhao and Zabaras, Nicholas and Koutsourelakis, Phaedon-Stelios and Perdikaris, Paris},
	year = {2019},
	pages = {56--81},
}

@article{lu2021deepxde,
author = {Lu, Lu and Meng, Xuhui and Mao, Zhiping and Karniadakis, George Em},
title = {DeepXDE: A Deep Learning Library for Solving Differential Equations},
journal = {SIAM Review},
volume = {63},
number = {1},
pages = {208-228},
year = {2021}
}

@article{karniadakisphysics-informed2021,
	title = {Physics-informed machine learning},
	volume = {3},
	copyright = {2021 Springer Nature Limited},
	language = {en},
	number = {6},
	journal = {Nature Reviews Physics},
	author = {Karniadakis, George Em and Kevrekidis, Ioannis G. and Lu, Lu and Perdikaris, Paris and Wang, Sifan and Yang, Liu},
	year = {2021},
	note = {Number: 6
Publisher: Nature Publishing Group},
	pages = {422--440},
}

@article{pang_fpinns_2019,
	title = {{fPINNs}: {Fractional} {Physics}-{Informed} {Neural} {Networks}},
	volume = {41},
	shorttitle = {{fPINNs}},
	number = {4},
	journal = {SIAM Journal on Scientific Computing},
	author = {Pang, Guofei and Lu, Lu and Karniadakis, George Em},
	year = {2019},
	note = {Publisher: Society for Industrial and Applied Mathematics},
	pages = {A2603--A2626},
}

@article{wu_comprehensive_2023,
	title = {A comprehensive study of non-adaptive and residual-based adaptive sampling for physics-informed neural networks},
	volume = {403},
	journal = {Computer Methods in Applied Mechanics and Engineering},
	author = {Wu, Chenxi and Zhu, Min and Tan, Qinyang and Kartha, Yadhu and Lu, Lu},
	year = {2023},
	pages = {115671},
}

@article{li_fourier_2021,
	title = {Fourier {Neural} {Operator} for {Parametric} {Partial} {Differential} {Equations}},
	publisher = {arXiv},
	author = {Li, Zongyi and Kovachki, Nikola and Azizzadenesheli, Kamyar and Liu, Burigede and Bhattacharya, Kaushik and Stuart, Andrew and Anandkumar, Anima},
	year = {2021},
	journal = {arXiv preprint arXiv:2010.08895 [cs, math]},
}

@article{yu_gradient-enhanced_2022,
	title = {Gradient-enhanced physics-informed neural networks for forward and inverse {PDE} problems},
	volume = {393},
	language = {en},
	journal = {Computer Methods in Applied Mechanics and Engineering},
	author = {Yu, Jeremy and Lu, Lu and Meng, Xuhui and Karniadakis, George Em},
	year = {2022},
	pages = {114823},
}

@article{yazdani_systems_2020,
	title = {Systems biology informed deep learning for inferring parameters and hidden dynamics},
	volume = {16},
	language = {en},
	number = {11},
	journal = {PLOS Computational Biology},
	author = {Yazdani, Alireza and Lu, Lu and Raissi, Maziar and Karniadakis, George Em},
	editor = {Hatzimanikatis, Vassily},
	year = {2020},
	pages = {e1007575},
}

@article{daneker_systems_2023,
	title = {Systems {Biology}: {Identifiability} {Analysis} and {Parameter} {Identification} via {Systems}-{Biology}-{Informed} {Neural} {Networks}},
	volume = {2634},
	shorttitle = {Systems {Biology}},
	language = {eng},
	journal = {Methods in Molecular Biology (Clifton, N.J.)},
	author = {Daneker, Mitchell and Zhang, Zhen and Karniadakis, George Em and Lu, Lu},
	year = {2023},
	pmid = {37074575},
	pages = {87--105},
}

@article{chen_physics-informed_2020,
	title = {Physics-informed neural networks for inverse problems in nano-optics and metamaterials},
	volume = {28},
	language = {eng},
	number = {8},
	journal = {Optics Express},
	author = {Chen, Yuyao and Lu, Lu and Karniadakis, George Em and Dal Negro, Luca},
	year = {2020},
	pmid = {32403669},
	pages = {11618--11633},
}

@article{Zhu2023FourierDeepONet,
  title={Fourier-DeepONet: Fourier-enhanced deep operator networks for full waveform inversion with improved accuracy, generalizability, and robustness},
  author={Zhu, Min and Feng, Shihang and Lin, Youzuo and Lu, Lu},
  journal={Computer Methods in Applied Mechanics and Engineering},
  volume={116300},
  year={2023},
  publisher={Elsevier}
}

@article{Zhang2019Quantifying,
  title={Quantifying total uncertainty in physics-informed neural networks for solving forward and inverse stochastic problems},
  author={Zhang, D. and Lu, L. and Guo, L. and Karniadakis, G. E.},
  journal={Journal of Computational Physics},
  volume={397},
  pages={108850},
  year={2019},
  publisher={Elsevier}
}

@article{Jiang2023FourierMIONet,
  title={Fourier-MIONet: Fourier-enhanced multiple-input neural operators for multiphase modeling of geological carbon sequestration},
  author={Jiang, Z. and Zhu, M. and Li, D. and Li, Q. and Yuan, Y. O. and Lu, L.},
  journal={arXiv preprint arXiv:2303.04778},
  year={2023}
}

\end{document}